\documentclass[sigconf]{aamas}  % do not change this line!

\AtBeginDocument{%
  \providecommand\BibTeX{{%
    \normalfont B\kern-0.5em{\scshape i\kern-0.25em b}\kern-0.8em\TeX}}}

\usepackage{booktabs}

\usepackage{flushend}
\setcopyright{ifaamas}  % do not change this line!
\acmDOI{doi}  % do not change this line!
\acmISBN{}  % do not change this line!
\acmConference[AAMAS'20]{Proc.\@ of the 19th International Conference on Autonomous Agents and Multiagent Systems (AAMAS 2020), B.~An, N.~Yorke-Smith, A.~El~Fallah~Seghrouchni, G.~Sukthankar (eds.)}{May 2020}{Auckland, New Zealand}  % do not change this line!

\usepackage{algorithmicx}
\usepackage{amsmath}
\usepackage{amsthm}
\usepackage{amsbsy}
\usepackage{amssymb}
\usepackage{amsfonts}
\usepackage[ruled,linesnumbered]{algorithm2e}
\usepackage{multirow}
\usepackage{graphicx}
\usepackage{subcaption}
\usepackage{epsfig}
\usepackage{epstopdf}
\usepackage{subcaption}
\usepackage{wrapfig}
\usepackage{float}
\usepackage{color}
\usepackage{url}
\usepackage{balance}
\usepackage{pdfpages}

\newcommand{\mc}{\mathcal}
\newcommand{\R}{\mathbb{R}}

\newcommand{\ol}{\overline}

\DeclareMathOperator*{\argmin}{argmin}

\graphicspath{figures/}

\newenvironment{proofsketch}{%
    \renewcommand{\proofname}{Proof Sketch}\proof}{\endproof}

\begin{document}
\fancyhead{}
\acmYear{2020}  % do not change this line!
\copyrightyear{2020}  % do not change this line!
\acmPrice{}  % do not change this line!

\title{
    Learning to Optimize Autonomy in Competence-Aware Systems
}

\author{Connor Basich}
\affiliation{%
 \institution{University of Massachusetts Amherst}
 \city{Amherst} 
 \state{Massachusetts}
}

\author{Justin Svegliato}
\affiliation{%
 \institution{University of Massachusetts Amherst}
 \city{Amherst} 
 \state{Massachusetts}
}

\author{Kyle Hollins Wray}
\affiliation{%
 \institution{Alliance Innovation Lab Silicon Valley}
 \city{Santa Clara}
 \state{California}
}

\author{Stefan Witwicki}
\affiliation{%
 \institution{Alliance Innovation Lab Silicon Valley}
 \city{Santa Clara}
 \state{California}
}

\author{Joydeep Biswas}
\affiliation{%
 \institution{The University of Texas at Austin}
 \city{Austin}
 \state{Texas}
}

\author{Shlomo Zilberstein}
\affiliation{%
 \institution{University of Massachusetts Amherst}
 \city{Amherst} 
 \state{Massachusetts}
}

\renewcommand{\shortauthors}{C. Basich et al.}

\begin{abstract}
Interest in \emph{semi-autonomous systems} (SAS) is growing rapidly as a paradigm to deploy autonomous systems in domains that require occasional reliance on humans.  This paradigm allows service robots or autonomous vehicles to operate at varying levels of autonomy and offer safety in situations that require human judgment. We propose an introspective model of autonomy that is learned and updated online through experience and dictates the extent to which the agent can act autonomously in any given situation. We define a \emph{competence-aware system} (CAS) that explicitly models its own proficiency at different levels of autonomy and the available human feedback. A CAS learns to adjust its level of autonomy based on experience to maximize overall efficiency, factoring in the cost of human assistance.  We analyze the convergence properties of CAS and provide experimental results for robot delivery and autonomous driving domains that demonstrate the benefits of the approach.
\end{abstract}

\keywords{\small
Markov decision processes; probabilistic planning; human-agent systems; adjustable autonomy; competence-aware systems}

\maketitle

\section{Introduction}

Recent progress in artificial intelligence and robotics has enabled the deployment of increasingly autonomous systems in complex unstructured domains such as
space exploration~\cite{Gao17,Mustard13}, autonomous underwater vehicles~\cite{Cashmore14,Kunz09,Sousa12},
service robots~\cite{Hawes17,Meeussen111}, and, most notably, autonomous vehicles~\cite{Broggi99,Broggi12,Dickmanns07}.
Common to these settings is that the autonomous systems are required to operate without supervision over extended periods of time while they rely on approximate models of the environment that may not be sufficient for handling every situation~\cite{Saisubramanian19}. Consequently, autonomous systems need to rely on various forms of human supervision and assistance~\cite{Coradeschi06,Zilberstein15}.

The vast majority of autonomous systems under development are in fact \emph{semi-autonomous systems} (SAS) that can operate autonomously under certain conditions, but may require human intervention or aid in order to achieve their assigned goals~\cite{Zilberstein15}.
For example, a space exploration rover may suspend operation and wait for a new plan from the command center when its wheel encounters unexpected resistance.
An autonomous car may request that the driver take over when lane demarcation is lost.  The reliance on humans in these situations is indicative of the limited competence of the autonomous system.  Human response could come in different forms that correspond to different limitations of the autonomous system.  For example, allowing a system to operate autonomously \emph{under human supervision} indicates a higher level of competence relative to a system that must first present its plan and get approval for every action before the action is executed.  

We propose to represent the competence of a semi-autonomous system using different levels of autonomy, each associated with distinct forms of human involvement.  Intuitively, we expect higher competence to imply that the necessary human involvement requires less effort or is less costly.  The resulting competence-aware system (CAS) can operate in multiple levels of autonomy, each of which is associated with different restrictions on autonomous operation and different forms of human assistance that compensate for the restricted abilities of the system. We further associate with each type of human assistance a unique set of feedback signals, the likelihood of which can be learned over time by the system. 

Determining the exact competence of an autonomous system at design time is very difficult, particularly when the environment is not fully specified.  For example, an autonomous vehicle may be initially authorized to operate in a fully autonomous mode on highways only during daytime and clear weather.
Hence, an initial level of competence could be determined during testing and evaluation, but adjustments must be made when the system is deployed.  Even when developers aim to err on the side of caution and define a lower level of autonomy as the default, it is also possible to unintentionally infer from initial testing that the system is more competent than it really is~\cite{Svegliato19, Rabiee2019}. Therefore, developing mechanisms to explicitly represent, reason about, and adjust the level of autonomy is an important challenge in artificial intelligence.  

Our objective in this paper is three-pronged: (1) to develop a formal \emph{representation of competence} using distinct levels of autonomy and distinct forms of human assistance, (2) to learn to \emph{optimize autonomy} based on human assistive actions, and (3) to develop a \emph{competence-aware} planning framework that factors in the system's knowledge about its own competence in order to reduce unnecessary reliance on humans.  Intuitively, the highest level of autonomy that a system can handle without human overrides captures its ``true competence.''  Our goal is to introduce learning mechanisms that allow the system to converge on its true competence and thereby minimize its reliance on human assistance.  We present a formal model and algorithms for implementing a CAS, a theoretical analysis of the model, and experimental results demonstrating the effectiveness of our approach in simulation as well as integration with an autonomous vehicle (AV) prototype.

\section{Related Work}

There has been substantial work on planning for semi-autonomous systems over the last two decades, particularly on various forms of \emph{adjustable autonomy}.  Adjustable autonomy refers to the ability of an autonomous system to alter its level of autonomy during plan execution, often by dynamically imposing or relaxing constraints on the extent of actions it can perform autonomously in a human-agent team~\cite{Dorais99,Bradshaw05,Mostafa19,Moffitt06,SZ:MZBJecai10,Vecht09,Scerri02,Zieba10}.
Our work adds two important capabilities to systems with adjustable autonomy: (1) explicitly modeling multiple forms of human feedback and the ability to learn about the agent's competence from feedback, and (2) learning predictive models of human feedback to allow the agent to converge on its level-optimal autonomy over time.

For autonomous vehicles in particular, SAE International has developed five distinct levels of autonomy that have become an industry standard~\cite{sae2014taxonomy}.  Our AV example uses levels of autonomy inspired by that standards and our experience with an AV prototype.  

The problem of safely transferring control from automation to a human in semi-autonomous systems~\cite{Zilberstein15} has been studied, particularly in the context of autonomous vehicles~\cite{Wray16}.  While such transfer mechanisms are highly relevant to our work, for the sake of clarity we do not model these transitions explicitly in this paper and instead focus on competence modeling and identifying the best level of autonomy in each situation.

Our work relates to broader research on \emph{long-term autonomy}~\cite{SZ:WSSWZlta18}, particularly \emph{symbiotic autonomy}~\cite{Coradeschi06,veloso2012cobots,veloso2015cobots} and \emph{human-in-the-loop AI} systems~\cite{zanzotto2019human}.
Symbiotic autonomy allows collaborative robots (CoBots) to proactively seek external help to overcome their limitations.  This work takes symbiotic autonomy to the next level, allowing agents to learn a model of their limitations from human feedback and consider the most cost-effective form of human assistance for each situation based on the acquired models.
Similarly, the issue of \emph{authority sharing} has been studied in human-robot systems~\cite{Mercier2010}, however in our case we assume that the human is always the authority, never the autonomous system.

\section{Problem Formulation}

We start with a description of the general problem.  Consider an autonomous agent that can operate in and plan for multiple levels of autonomous operation, each of which consists of different forms of human feedback.
In particular, this paper focuses on agents that use the following three models: a \textit{domain model} (DM) that describes the environment that the agent is operating in, an \textit{autonomy model} (AM) that describes the levels of autonomy the agent can operate in, where it is allowed to do so, and what the respective utilities are, and a \textit{human feedback model} (HM) that describes the types of feedback that the agent can receive from the human, how costly each type of feedback is, and how likely the agent is to receive it.
Figure~\ref{fig:competence_aware_system} represents an overview of a competence-aware system with specific levels of autonomy and feedback signals that we use throughout the rest of the paper as a running example.

\subsection{Domain Model}

The \textit{domain model} (DM) describes the environment in which the agent operates, most notably the transition and cost dynamics of the environment with respect to the agent. In this paper, we model this as a Stochastic Shortest Path (SSP) problem, a formal decision-making model for reasoning in stochastic environments where the objective is to find the least-cost path from a start state to a goal state. An SSP is a tuple $\langle S, A, T, C, s_0, s_g \rangle$, where $S$ is a finite set of states, $A$ is a finite set of actions, $T : S \times A \times S \rightarrow [0,1]$ represents the probability of reaching state $s' \in S$ after performing action $a \in A$ in state $s \in S$, $C : S \times A \rightarrow \mathbb{R}^+$ represents the expected immediate cost of performing action $a \in A$ in state $s \in S$, $s_0$ is an initial state, and $s_g$ is a goal state such that $\forall a \in A, T(s_g, a, s_g) = 1 \land C(s_g, a) = 0$.

A solution to an SSP is a policy $\pi : S \rightarrow A$ that indicates that action $\pi(s) \in A$ should be taken in state $s \in S$. A policy $\pi$ induces the value function $V^{\pi} : S \rightarrow \mathbb{R}$ that represents the expected cumulative cost $V^{\pi}(s)$ of reaching $s_g$ from state $s$ following the policy $\pi$. An optimal policy $\pi^*$ minimizes the expected cumulative cost $V^*(s_0)$ from the initial state $s_0$.

\subsection{Autonomy Model}

The \textit{autonomy model} (AM) captures the extent of autonomous operation that the agent can perform, i.e., both the actual different forms of autonomous operation as well as when each is allowed by some external constraint.
We formally represent this by the tuple $\langle \mc{L}, \kappa, \mu \rangle$, where $\mc{L} = \{ l_0, ..., l_n \}$ is the set of levels of autonomy where each level $l_i$ corresponds to some set of constraints on the system's autonomous operation. 
The constraints are reflected by the form of human involvement; for example, in supervised autonomy a person must monitor the system and override any actions deemed unsafe or undesirable.  We do not restrict the number of levels and forms of human involvement allowed in an autonomy model. 
Intuitively, the higher the level of autonomy, the lower the cost of human involvement, although that is not a requirement of the autonomy model.
 
Without loss of generality, we assume that $\mc{L}$ is a fully ordered set, although in general the theory extends to any graph in which two levels are connected when the level of autonomy could change from one to the other.
In Figure~\ref{fig:competence_aware_system}, these levels correspond to \textit{no autonomy}, where the agent needs a human to perform the action manually, \textit{verified autonomy}, where the agent must query for and receive explicit approval before even attempting the action, \textit{supervised autonomy}, where the agent can perform the action autonomously as long as there is a human supervising the agent who can intervene if something is going wrong, and \textit{unsupervised autonomy}, that represents fully autonomous operation. 

Next, $\kappa : S \times A \rightarrow \mc{P}(\mc{L})$ is the \textit{autonomy profile} mapping states $s \in S$ and actions $a \in A$ to a subset of $\mc{L}$ (note that $\mc{P}(\mc{L})$ denotes the powerset of $\mc{L}$), prescribing constraints on the allowed levels of autonomy for any situation.
These can be hard constraints on the system (i.e. technical, legal, or ethical) or can be temporary conservative constraints that can be updated over time as the system improves.
In Figure~\ref{fig:competence_aware_system}, $\kappa$ constrains the space of all policies $\Pi$ so that the system is only allowed to follow a policy that never violates $\kappa$. Finally, $\mu : S \times \mc{L} \times A \times \mc{L} \rightarrow \R$ represents the \emph{cost of autonomy} of performing action $a \in A$ at level $l' \in \mc{L}$ given that the agent is in state $s \in S$ and just operated in level $l \in \mc{L}$ in the previous state.

\subsection{Human Feedback Model}

The \textit{human feedback model} (HM) describes the agent's knowledge about, and predictions of, its interactions with the human.
We formally represent this as the tuple $\langle \Sigma, \lambda, \rho, \tau \rangle$, where $\Sigma = \{ \sigma_0, ..., \sigma_n \}$ is the set of possible feedback signals the agent can receive from the human, $\lambda : S \times \mc{L} \times A \times \mc{L} \rightarrow \Delta^{|\Sigma|}$ is the \textit{feedback profile} that represents the probability of receiving signal $\sigma$ when performing action $a \in A$ at level $l' \in \mc{L}$ given that the agent is in state $s \in S$ and just operated in level $l \in \mc{L}$, $\rho : S \times \mc{L} \times A \times \mc{L} \rightarrow \R^+$ is the \textit{human cost function} and represents the positive cost to the human of performing action $a \in A$ at level $l' \in \mc{L}$ given that the agent is in state $s \in S$ and just operated in level $l \in \mc{L}$, and $\tau: S \times A \rightarrow \Delta^{|S|}$ is the \textit{human state transition function} that represents the probability of the human taking the agent to state $s' \in S$ when the agent attempted to perform action $a \in A$ in state $s \in S$ but the human took over control. 

In practice, the feedback profile $\lambda$ and the human state transition function $\tau$ are assumed to be unknown a priori, so the agent must estimate them based on previous data it has gathered in the same or similar situations.
In Figure~\ref{fig:competence_aware_system}, after the action execution stage, the system will record the feedback it receives from the human, if any, and use that to update these model components. In practice, the feedback signals may also not be instantaneous, and in some cases could require a complex process of transferring control to and from a human over the course of an indefinite amount of time, where elements of the transfer process such as the communication interface are individually important. The problem of transfer of control in semi-autonomous systems has been separately studied~\cite{Wray16}, however for the sake of clarity we do not model this process explicitly in this work as we are focusing on the orthogonal problem of levels of autonomy and competence modeling.

\begin{figure}
    \centering
    \includegraphics[width=\linewidth]{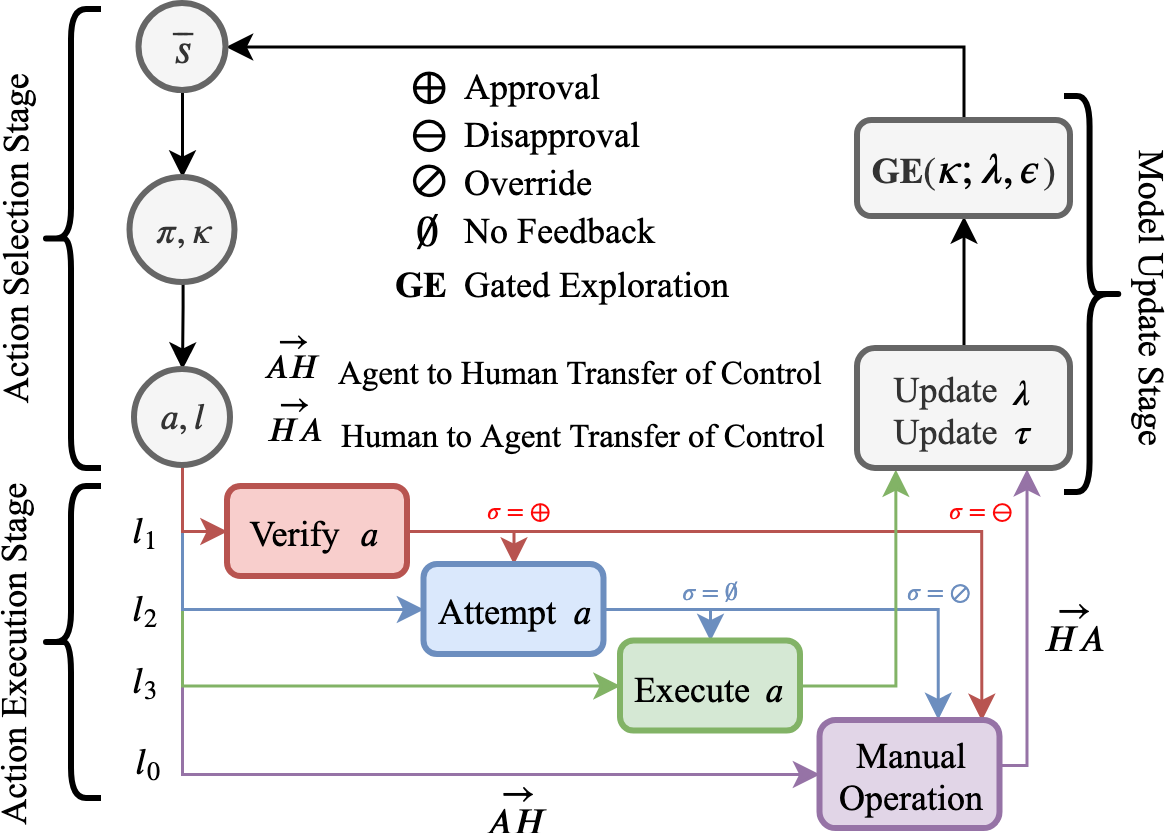}
    \caption{A competence-aware system with four levels of autonomy---verified, supervised, unsupervised, and no autonomy---and four type of feedback signals---approval, disapproval, override, and no feedback.}
    \label{fig:competence_aware_system}
\end{figure}

\section{Competence-Aware Systems}

Given the three model components above, we can now define a competence-aware system (CAS) that combines the three models into one formal decision-making framework.
A CAS therefore represents a planning problem that accounts for the different levels of autonomy available to the agent and factors the agent's expectations regarding the likelihood and cost of human intervention.  The objective of a solution to a CAS planning problem is to create a cost-effective plan that minimizes the cost of reaching the goal, including the cost of human assistance.  
Hence, the CAS uses the autonomy model to proactively generate plans that operate across multiple levels of autonomy, by leveraging the human feedback model to predict the likelihood of different feedback signals to optimize the level of autonomy and minimize the reliance on humans. 

\subsection{Model Definition}
CAS combines the domain model, autonomy model, and human feedback model into one sequential decision-making framework where the objective is to generate a policy that minimizes the expected cost of accomplishing a task. Hence we formally represent CAS as the following extension of an SSP:

\begin{definition}
    A \textbf{competence-aware system} $\mc{S}$ is represented by the tuple
    $\langle \ol{S}, \ol{A}, \ol{T}, \ol{C}, \ol{s}_0, \ol{s}_g \rangle$, where:
    \begin{itemize}
        \item $\ol{S} = S \times \mc{L}$ is a set of factored states such that $S$ is the set of domain states and $\mc{L}$ is the levels of autonomy.
        \item $\ol{A} = A \times \mc{L}$ is a set of factored actions such that $A$ is the set of domain actions and $\mc{L}$ is the levels of autonomy.
        \item $\ol{T} : \ol{S} \times \ol{A} \rightarrow \Delta^{|\ol{S}|}$ is a transition function comprised of a state transition function $T_l : S \times A \rightarrow \Delta^{|S|}$ for each level $l \in \mc{L}$.
        \item $\ol{C} : \ol{S} \times \ol{A} \rightarrow \R^+$ is a positive cost function comprised of $C: S \times A \rightarrow \R^+$, $\mu : \ol{S} \times \ol{A} \rightarrow \R$, and $\rho : \ol{S} \times \ol{A} \rightarrow \R^+$.
        \item $\ol{s}_0 \in \ol{S}$ is the initial state such that $\ol{s}_0 = \langle s_0, l \rangle$ for some $l \in \mc{L}$.
        \item $\ol{s}_g \in \ol{S}$ is the goal state such that $\ol{s}_g = \langle s_g, l \rangle$ for some $l \in \mc{L}$.
    \end{itemize}
\end{definition}

A solution to a given CAS is a policy $\pi$ that maps states and levels $\ol{s} \in \ol{S}$ to actions and levels $\ol{a} \in \ol{A}$, where the space of policies that the agent can consider is restricted by the autonomy profile $\kappa$ in the following way. 

Let $\ol{a} = \langle a, l \rangle$. Given $\ol{s} = \langle s, l' \rangle\in \ol{S}$, we say that $(\ol{s}, \ol{a})$ is \textbf{allowed} if $l \in \kappa(s, a)$, and a policy $\pi$ is allowed if for every $\ol{s} \in \ol{S}$, $(\ol{s}, \pi(\ol{s}))$ is allowed. We denote the set of policies $\pi \in \Pi$ that are allowed given $\kappa$ as $\Pi_\kappa \subseteq \Pi$ and require that any policy returned by solving the CAS, $\pi^*$, is always taken from $\argmin_{\pi \in \Pi_\kappa}V^\pi(s_0)$.

\subsection{Sample CAS}

As noted earlier, we focus on the CAS illustrated in Figure~\ref{fig:competence_aware_system}, which represents a class of CAS with four distinct levels of autonomy and four feedback signals.
Recent work on autonomous vehicles~\cite{Wray17} and autonomous mobile robots~\cite{Svegliato19} suggests that this class of CAS represents a wide range of autonomous systems. In Figure~\ref{fig:competence_aware_system}, the policy $\pi$ produces an action $a$ and a level $l$ for every state $\ol{s}$. $l$ dictates the manner in which the system carries out the action $a$, and the autonomy profile $\kappa$ restricts the levels the $\pi$ can return.

Formally, let $\mc{L} = \lbrace l_0, l_1, l_2, l_3 \rbrace$ where
\begin{itemize}
    \item $l_0$ is \emph{no autonomy}, which requires direct human aid in the form of manual control.
    \item $l_1$ is \emph{verified autonomy}, which requires the agent to query for and receive human approval prior to executing the action.
    \item $l_2$ is \emph{supervised autonomy}, which requires a human to be present and available to intervene in the case of failure.
    \item $l_3$ is \emph{unsupervised autonomy}, which involves no human in the loop at all.
\end{itemize}

\noindent Let $\Sigma = \lbrace \emptyset, \oplus, \ominus, \oslash \rbrace$, corresponding to \emph{no feedback}, \emph{approval}, \emph{disapproval}, and \emph{override} respectively. Furthermore, we assume that $\oplus$ and $\ominus$ can only be received in $l_1$, and $\oslash$ and $\emptyset$ only in $l_2$.

We can now specify the \textbf{state transition function} of this CAS. Given 
$\ol{s}, \ol{a},$ and $\ol{s}'$, we define $\ol{T}$ as follows:
\setlength{\belowdisplayskip}{1pt} \setlength{\belowdisplayshortskip}{1pt}
\setlength{\abovedisplayskip}{1pt} \setlength{\abovedisplayshortskip}{1pt}
\begin{equation}
        \ol{T}(\ol{s}, \ol{a}, \ol{s}') = 
        \begin{cases}
            \tau(s, a, s'), & \text{if $l = l_0$},\\
            \lambda( \oplus ) T(s, a, s')+ \lambda( \ominus ) [s = s'], & \text{if $l = l_1$},\\
            \lambda( \emptyset ) T(s, a, s')+ \lambda( \oslash ) \tau(s, a, s'), & \text{if $l = l_2$},\\
            T(s, a, s'), & \text{if $l = l_3$},
        \end{cases}
        \label{eq: T-bar}
\end{equation}
where $\lambda(\cdot) = \lambda(\cdot | \ol{s}, \ol{a})$ and $[\cdot]$ denotes Iverson brackets.

We further define $\ol{C}(\ol{s}, \ol{a})$ as follows:
\setlength{\belowdisplayskip}{1pt} \setlength{\belowdisplayshortskip}{1pt}
\setlength{\abovedisplayskip}{1pt} \setlength{\abovedisplayshortskip}{1pt}
\begin{equation}
    \ol{C}(\ol{s}, \ol{a}) = g\big(C(s, a), \mu(\ol{s}, \ol{a}), \rho(\ol{s}, \ol{a})\big).
\end{equation}
where $g$ is any cost aggregation function on $C, \mu$, and $\rho$, the simplest case of which is a weighted summation of the three values.

Intuitively, Eqn.~\ref{eq: T-bar} states that when the agent operates in $l_0$, it follows the transition dynamics of the human who takes control. When operating in $l_1$, the probability it arrives in state $s'$ is the probability it is approved to take the action times the probability that it succeeds following $T$ plus the probability that it is disapproved and the state is the same. In $l_2$, the probability it arrives in state $s'$ is the probability it succeeds following $T$ without any human intervention plus the probability that the human overrides it and takes it to that state. When the agent operates in $l_3$, it follows the transition dynamics of the domain model DM.

\subsection{Gated Exploration}
A fundamental component of the CAS model is the ability to adjust its autonomy profile over time using what the system has learned, and to optimize its autonomy by reducing unnecessary reliance on human assistance. However, before operating in a new level of autonomy, the system may have no knowledge of how the human will interact with it in that level, i.e., the feedback profile in that new level may be initialized by default to some baseline distribution. As a result, it is necessary that the system \emph{explore} levels of autonomy that it has reason to believe may be more cost effective than its current level, so that it may generate the data it needs to improve the accuracy and confidence of its feedback profile in those levels. 

However, allowing a system to alter its own autonomy profile
can lead to severe consequences in the real world if not done carefully. Therefore, we propose an extension to traditional exploration methods used in reinforcement learning called \emph{gated exploration}, in which the system must obtain permission from a human before exploring a new (disallowed) level of autonomy.  Hence, the system must first query the human to update the autonomy profile to allow such exploration. This way, the exploration of disallowed levels is \emph{gated} by a human authority to prevent the agent from randomly executing dangerous actions. 

Although many exploration-exploitation strategies may work in our context, we use a variant of $\epsilon$-greedy where $\epsilon$ is not fixed but instead proportional to the relative expected cost of performing a given action in each level of autonomy. More formally, the probability of exploring a level $l'$ adjacent to the current level $l$ in $\mc{L}$ is proportional to the softmax of the negative $q$-value of operating in level $l'$ over all levels adjacent to $l$. 

\subsection{Autonomy Profile Initialization}

Because we restrict the system to choose policies from $\Pi_\kappa$, if the autonomy profile $\kappa$ is altered, so too is the space of allowed policies; in particular this means that the optimal policy is, intuitively, only as good as $\kappa$. Hence there is a trade-off when setting the initial constraints on the allowed autonomy of the system, i.e., $\kappa$.

One can take a conservative approach and constrain the system significantly, for instance setting $|\kappa(s,a)| = 1$
so that a single level is deterministically selected 
for every $(s,a) \in S \times A$, reducing the problem complexity to solving the underlying domain model. However, doing so risks a globally suboptimal policy with respect to $\mc{L}$ and may, depending on the initial $\kappa$, make reaching the globally optimal policy impossible. On the other extreme, one can take a risky approach and not constrain the system at all a priori, leaving the decision of choosing the level of autonomy completely up to the system. This approach, while necessarily containing the optimal policy (subject to the agent's model) is naturally slower due to the larger policy space and inherently less safe as the agent can take actions in undesirable levels.

We propose that in practice, the ideal initialization is somewhere in the middle; $\kappa$ should be less constraining in situations where the expected cost of failure is relatively low, and more constraining in situations where it is high. For instance, in an autonomous vehicle, $\kappa$ should be more constraining initially in situations involving pedestrians, poor visibility, or chaotic environments such as large intersections with multiple vehicles; however, initial testing may indicate that driving along a highway is low-risk and may not require a highly constraining $\kappa$.

\section{Theoretical Analysis}
\label{sec:theory}
In this section we first state several key properties of a CAS and prove the following claims:

First, we show that under standard convergence assumptions, the feedback profile $\lambda$ will converge to the human's true feedback distribution. Second, we show that \emph{if feasible}, the agent's policy $\pi$ will converge to its competence regardless of how $\lambda$ is initialized, given a reasonable initialization of $\kappa$.

\subsection{Properties of a CAS}
\label{sec:properties}
We will refer to the human authority henceforth by the notation $\mc{H}$, and we make the following assumptions about them. First, we assume that the human authority is \emph{$\epsilon$-consistent}, which means that given two identical queries, the probability that they respond differently given no new information is bounded by some $0 < \epsilon \ll 1$. Second, we begin by assuming that the human authority's feedback signals come from an underlying stationary distribution, $\lambda^\mc{H}$, which may not be known even by $\mc{H}$ a priori, and is hence Markovian.

We now define three central properties of a CAS.

\begin{definition}
\label{def:lambda_h}
    Let $\lambda^{\mc{H}}$ be the stationary distribution of feedback signals that the human authority follows.
\end{definition}

\begin{definition}
\label{def:kappa_h}
    Let $\kappa^{\mc{H}} : S \times A \rightarrow \mc{P}(\mc{L})$ be a mapping that maps state-action pairs to the set of levels of autonomy the human authority will allow the system to operate in with nonzero probability. 
\end{definition}
Intuitively, $\kappa^{\mc{H}}$ represents the human authority's belief of the \emph{agent's} underlying competence. Hence by definition, any level that is not contained in $\kappa^\mc{H}(s,a)$ will never be approved by the human.

\begin{definition}
\label{def:competence}
    Let $\mc{S}$ be a CAS. The \emph{competence of $\mc{S}$}, denoted $\chi_\mc{S}$, is a mapping from $\ol{S} \times A$ to the optimal (least-cost) level of autonomy given perfect knowledge of $\lambda^{\mc{H}}$. Formally:
    \begin{equation}
        \chi_\mc{S}(\ol{s}, a) = \argmin_{l \in L} q(\ol{s}, (a, l) ; \lambda^{\mc{H}}).
    \end{equation}
\end{definition}

Intuitively, the system's level of competence for executing action $a$ in state $\ol{s}$ is the most beneficial (cost effective) level of autonomy were it to know exactly the human feedback model. In general, we expect this to be equal to $sup(\kappa^{\mc{H}}(s,a))$, i.e., the highest level of autonomy \emph{allowed} by the human, although it need not be the case always.  In principle, the highest allowed level of autonomy could require more frequent human interventions that may render it less efficient overall relative to a lower level of autonomy.  That is why we define the optimal level of autonomy based on the comprehensive expected cost.

It is important to note that this definition of competence relies on $\lambda^\mc{H}$, and hence is a definition of competence on \emph{the overall human-agent system}, and is explicitly not a measure of the competence of the underlying agent's fundamental abilities. A corollary of this fact is that the CAS is only as competent as the human authority believes it to be; a human authority that has a poor understanding of the system's abilities could lead to the system having a lower competence than a human authority that knows perfectly the limitations and capabilities of the system.

We now define two final properties of a CAS.
\begin{definition}
    Let $\mc{S}$ be a CAS. $\mc{S}$ is $\lambda$\textbf{-stationary} if for every state $\ol{s} = (s, l) \in \ol{S}$, and every action $a \in A$, the expected value of sample information (EVSI) on the feedback signals $\Sigma$ for $(\ol{s},a)$ is less than $\epsilon$ for any $\epsilon$ greater than 0.
\end{definition}

Intuitively, $\mc{S}$ is $\lambda$-stationary if, in expectation, any new feedback drawn from the true distribution $\lambda_\mc{H}$ will not affect $\lambda$ enough to change the optimal level of autonomy for any $\ol{s} \in \ol{S}$ and $a \in A$. 

\begin{definition}
    Let $\mc{S}$ be a CAS. $\mc{S}$ is \textbf{level-optimal} if for every state $\ol{s} = (s,l) \in \ol{S}$, we have that $\pi^*(\ol{s}) = \ol{a} = (a, \chi_\mc{S}(s, a))$. Similarly, $\mc{S}$ is $\epsilon$-\textbf{level-optimal} if this holds for $1-\epsilon$ percent of states.
\end{definition}

\subsection{Theoretical Results}
\label{sec:theory_results}

In this section, we show that under assumptions 1 and 2 as stated in Section~\ref{sec:properties}, if $\chi_\mc{S}(\ol{s},a) \in \kappa^{\mc{H}}(s,a)$ and there is a path from some level of autonomy in the system's initial $\kappa(s,a)$ to $\chi_\mc{S}(\ol{s},a)$ for every $\ol{s} \in \ol{S}, a \in A$, then the system's autonomy profile $\kappa$ will converge to its competence regardless of the initialization of $\lambda$.

\begin{proposition}
    Let $\lambda_t$ be the feedback profile after $t$ pieces of feedback have been received.
    As  $t \rightarrow \infty$, if no $(\ol{s},a)$ is starved and $\lbrace \lambda_t \rbrace$ converges in distribution, the autonomy-cognizant system will converge to $\lambda$-stationarity.
    \label{prop: stationarity-lambda}
\end{proposition}
\begin{proof}
    Let $\ol{s} \in \ol{S}$ and $a \in A$.
    As $\ol{s}$ and $a$ are arbitrary and we assume that no $(\ol{s},a)$ is starved, it is sufficient to show that $(\ol{s},a)$ converges to stationarity as $t \rightarrow \infty$.
    First, let $U(\lambda, l)$ be the q-value of $(\ol{s},a)$ under the optimal policy given that our feedback profile is $\lambda$ and we execute $a$ in level $l$.
    Then
\[ \text{EVSI} {=} \sum_{\sigma \in \Sigma} \max_{l \in L} \int_{\Lambda} \!U(\lambda,l)\lambda(\sigma|s,a,l)p(\lambda)d\lambda 
                    \,-\,  \max_{l \in L} \int_{\Lambda} \!U(\lambda,l)p(\lambda)d\lambda. \]

    Because $\lbrace \lambda_t \rbrace$ converges in distribution, $\lim_{t \rightarrow \infty} Pr(|\lambda_t - \lambda^\mc{H}| > \epsilon) = 0$ $\forall \epsilon>0$ where $\lambda^\mc{H}$ is the true distribution.
    Therefore, in the limit the probability that $\lambda = \lambda^\mc{H}$ after $t$ steps, $p_t(\lambda)$, defines a Dirac delta function with point mass centered at $\lambda^\mc{H}$.
    Hence we get that,
   \begin{align*} 
        \lim_{t \rightarrow \infty}&\text{EVSI} \\
         &= \big(\lim_{t \rightarrow \infty}\sum_{\sigma \in \Sigma}\max_{l \in L} \int_{\Lambda} U(\lambda,l)\lambda(\sigma|s,\emptyset,a,l)p_t(\lambda)d\lambda\big)\\
        &\hspace{6mm}- \big(\lim_{t \rightarrow \infty}\max_{l \in L} \int_{\Lambda} U(\lambda,l)p_t(\lambda)d\lambda\big)\\
        &= \big(\sum_{\sigma \in \sigma}\max_{l \in L}U(\lambda^\mc{H},l)\lambda^\mc{H}(\sigma|s,\emptyset,a,l)\big) - \big(\max_{l \in L} U(\lambda^\mc{H},l)\big)\\
        &= \sum_{\sigma \in \Sigma} \max_{l \in L} U(\lambda^\mc{H},l)(1 - \lambda^\mc{H}(\sigma|s,\emptyset,a,l))\\
        &= \max_{l \in L} U(\lambda^\mc{H},l)\big(1 - \sum_{\sigma \in \Sigma}\lambda^\mc{H}(\sigma|s,\emptyset,a,l)\big)\\
        &= \max_{l \in L} U(\lambda^\mc{H},l)(1-1)\\
        &= 0. \\[-22pt]
    \end{align*}
\end{proof}

\begin{theorem}
    Let $\mc{S}$ be a CAS that follows the gated exploration strategy for which there exists at least one path from $\kappa_0(s,a)$ to $\chi_\mc{S}(\ol{s}, a)$ in $\mc{L}$ where all levels along the path are in $\kappa^\mc{H}(s, a)$ for every $(\ol{s}, a) \in \ol{S} \times A$. Then given any initial $\lambda_0$, if no $(\ol{s}, a)$ is starved and $\{\lambda_t\}$ converges in distribution, then as $t \rightarrow \infty$, $\mc{S}$ will converge to level-optimality.
\label{thm: main}
\end{theorem}

\begin{proofsketch}
    By proposition \ref{prop: stationarity-lambda}, under the conditions as stated, $\lambda$ will converge to $\lambda^\mc{H}$. As a result, what is left is to show that in the limit, $\pi^*(\ol{s}) = (a, \kappa(s,a))$ for every $\ol{s} \in \ol{S}$. Because $\lambda$ has converged to $\lambda^{H}$, the system can determine the cost-optimal level of autonomy for every action $a$ in any state $\ol{s}$; this is exactly $\chi_\mc{S}(\ol{s},a)$. Hence, we must only show that the system will reach this level under the conditions stated. By the exploration policy, the system has a nonzero chance of exploring all neighboring levels at any given point in time, and since $\mc{L}$ is assumed to have a valid path from $\kappa_0(s,a)$ to $\chi\mc{S}(\ol{s},a)$, there is always a nonzero probability of reaching the optimal level via exploration. At this point, since $\lambda$ is converged, exploration will terminate. Hence, we are done.
\end{proofsketch}

\subsection{Model Assumptions}
\label{sec:assumptions}
We make two assumptions about the human authority, $\mc{H}$: (1) that the human provides consistent feedback, and (2), that the human's feedback comes from a stationary, Markovian distribution.  We discuss below practical considerations regarding these assumptions.

Regarding assumption (1), implicit in this assumption is that humans respond appropriately to each situation, possibly with some noise representing the likelihood of human error. 
However, because of the limited scope of the system's domain model, it could be that perfectly consistent feedback from $\mc{H}$'s perspective is \emph{perceived} to be random by the system, particularly when it is not aware of the domain features that explain the human feedback.  As an example, consider a robot that can open `push' doors, but cannot open `pull' doors. If the robot cannot discriminate between these types of doors, consistent and correct human feedback (approving autonomously opening `push' doors only) may be perceived by the robot to be arbitrary or random.  
Although in practice one may wish to avoid such situations, we emphasize that the system \emph{will still converge to its competence} -- possibly low competence -- when the feedback distribution appears to be random.

Regarding assumption (2) that the human feedback distribution $\lambda^\mc{H}$ is stationary and Markovian from the start, it implies that the human has good knowledge of the system from the start.  That may not be realistic.  It is more likely that the feedback signals may very based upon the observed performance of the system over time. However, as the human authority observes the system's performance, their feedback distribution will eventually reach a stationary point as long as the system's underlying capabilities stay fixed. Therefore, even if there are erroneous feedback signals provided early in this process, in the limit the system \emph{will still converge to its competence}.  Furthermore, two possible means of expediting this is to introduce a training phase at the beginning of the system's deployment to allow the human to observe the system's performance and develop accurate expectations regarding the system's capabilities, and to introduce standardized feedback criteria that is made known to the humans.

\begin{figure}[t!]
    \centering
    \includegraphics[width=0.8\columnwidth]{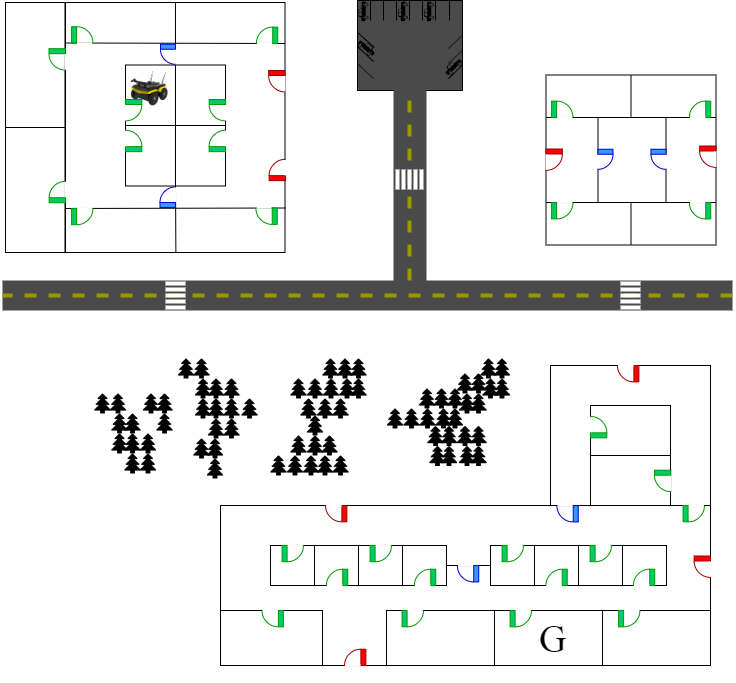}
    \caption{Campus delivery robot domain. The robot starts in a room on campus and must traverse both crosswalks and doors to reach its goal state.
In some cases, the robot may learn to go around a nearby parking lot rather than take the crosswalks due to poor visibility and learn to identify doors that can be opened autonomously.}
    \label{fig:campus_delivery_domain}
\end{figure}

\section{Experimental Results}
\label{sec: experiments}

We implemented our model in two different problem domains. The first domain features a campus delivery robot that must learn to properly navigate a large area featuring several kinds of obstacles. The second domain features an autonomous vehicle that must learn to properly pass an obstruction in its lane. We begin with a description of these domains as implemented in our experiments.

\emph{Campus Delivery Robot}  A delivery robot must navigate a campus to deliver packages from one location to another. An illustration of the map used in our experiments can be seen in Figure~\ref{fig:campus_delivery_domain}. In this domain there are two categories of potential obstacles for the robot. The first is a crosswalk that can have differing levels of traffic -- no traffic, light traffic, and heavy traffic. Although not pictured, certain crosswalks have better visibility than others; however, this is not modeled by the robot and so it must learn to discriminate using only human feedback. The second kind of obstacle is a door that the robot must get through to enter buildings, rooms, or even different areas of buildings once inside. Doors come in three colors -- blue, green, and red.  While the robot can see the colors of the doors, the meaning of these colors changes from building to building and is unknown to the robot a priori, so the system must once again rely on human feedback to discriminate between doors it is allowed to open and and those that is not, in each building on the campus. Trees, walls, and roads are all avoided by the system completely. $\kappa(s,a)$ is initialized to be $\{0,1\}$ for all actions in obstacle states. For states with no obstacles, the system is allowed to operate at level 3.

\begin{figure*}[t]
    \centering
    \begin{subfigure}[t]{0.33\textwidth}
        \centering
        \includegraphics[width=\textwidth]{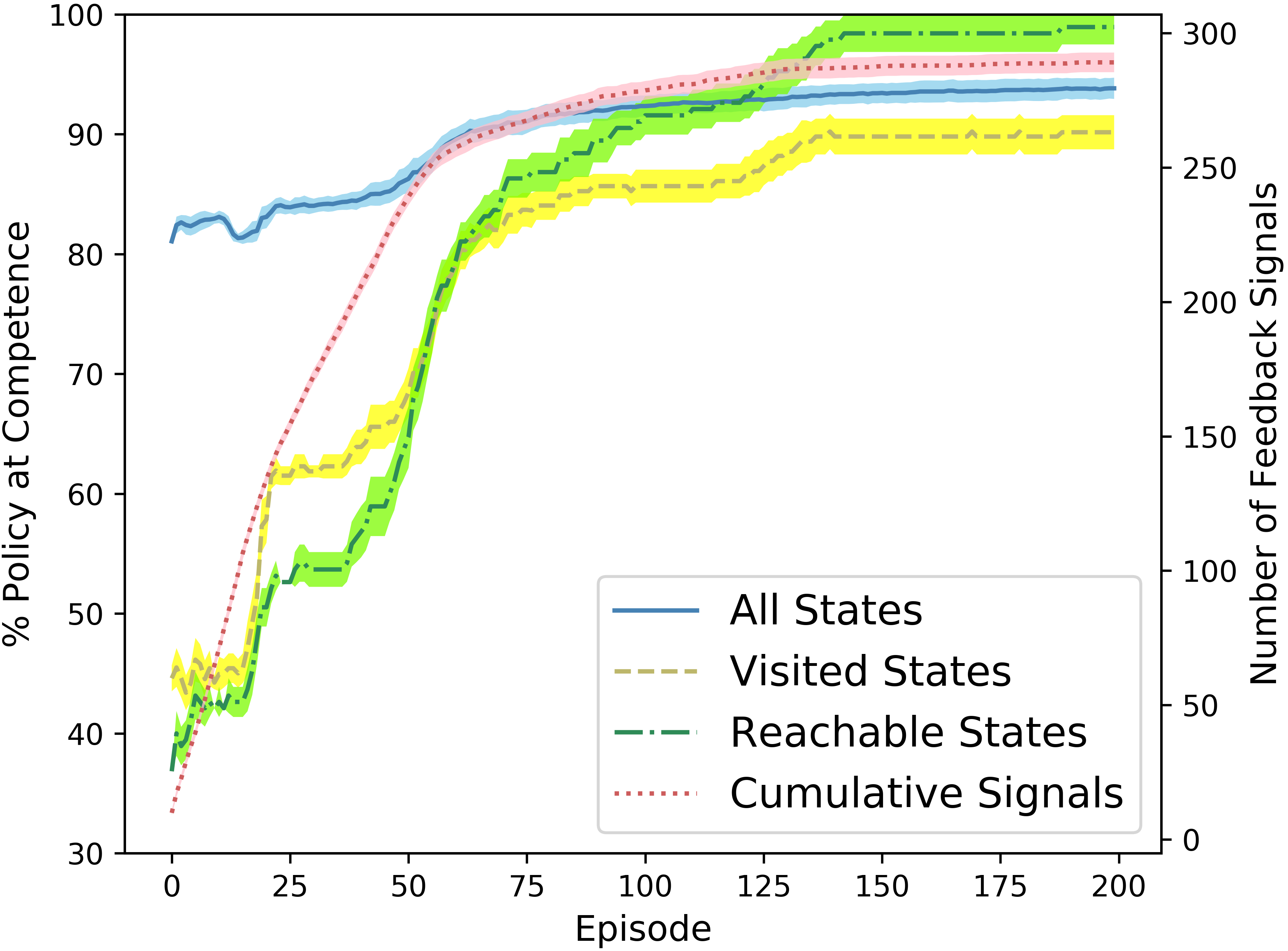}
        \caption{Campus delivery robot with a single task.}
        \label{fig:single_task_competence}
    \end{subfigure}
    \begin{subfigure}[t]{0.33\textwidth}
        \centering
        \includegraphics[width=\textwidth]{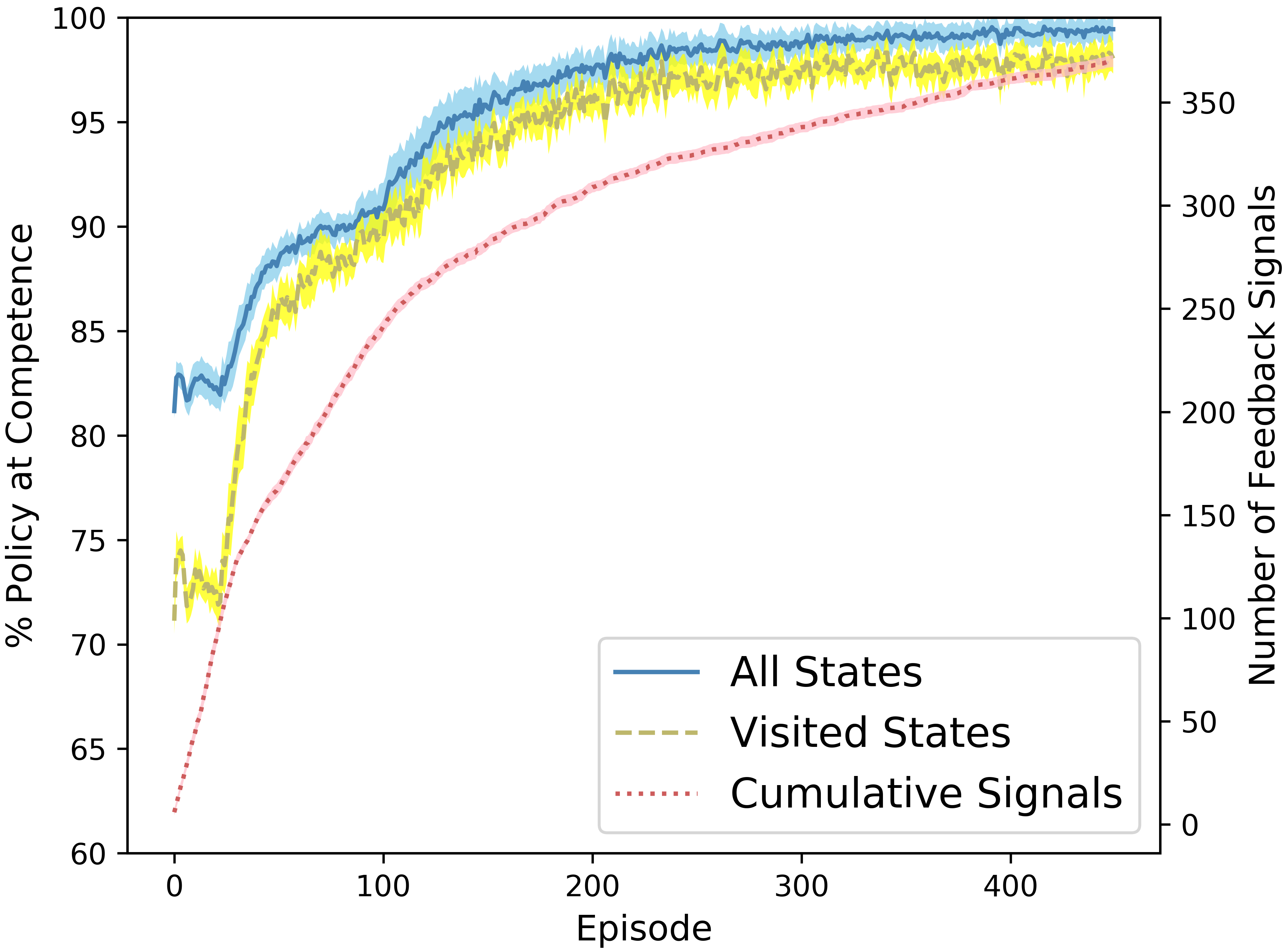}
        \caption{Campus delivery robot with random tasks.}
        \label{fig:random_task_competence}
    \end{subfigure}
    \begin{subfigure}[t]{0.33\textwidth}
        \centering
        \includegraphics[width=\textwidth]{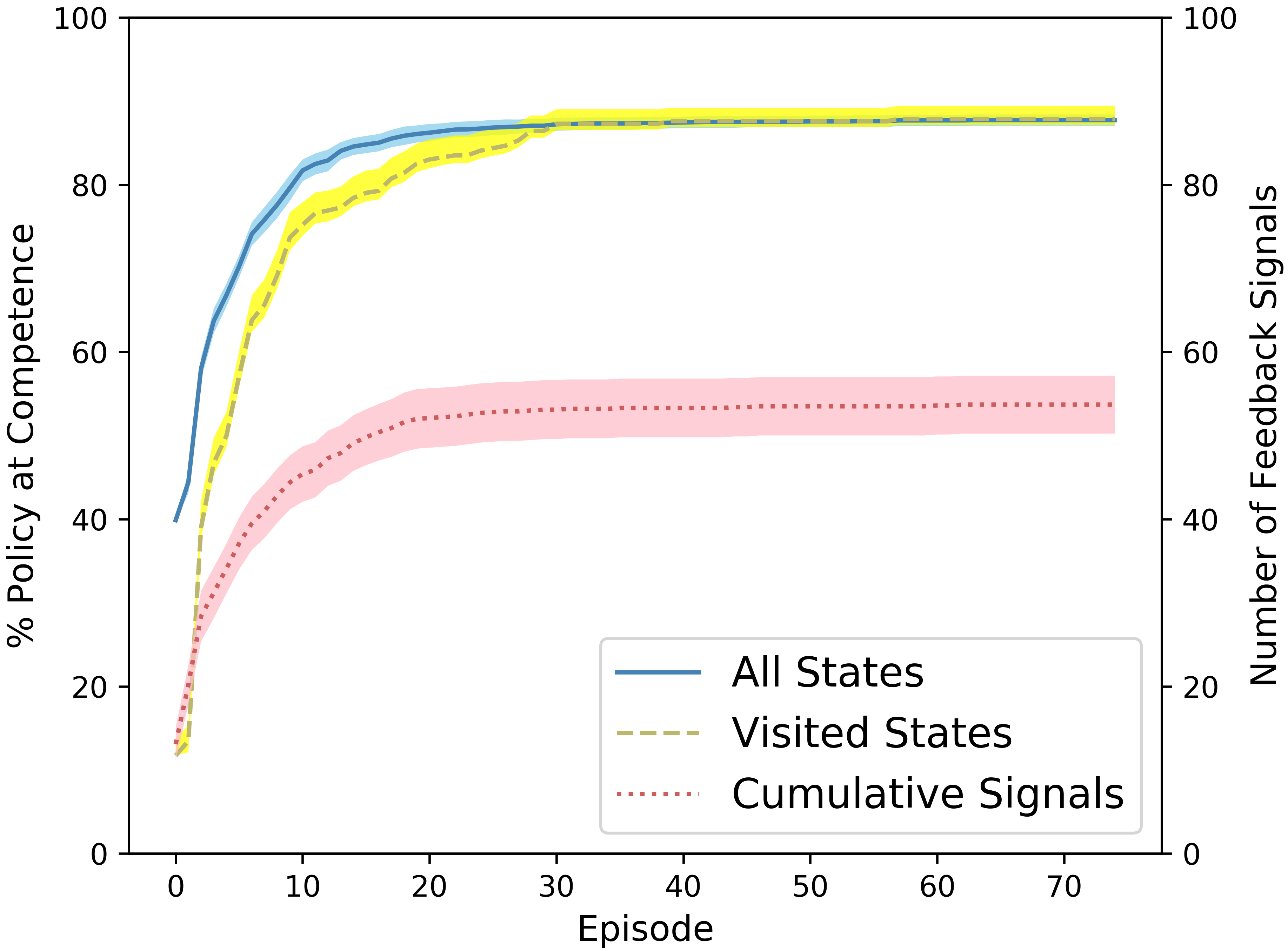}
        \caption{Autonomous vehicle obstacle passing.}
        \label{fig:obstacle_passing_competence}
    \end{subfigure}
    \vspace{-8pt}
    \caption{Level-optimality of the CAS across different subsets of the state space. \emph{All states} refers to the entire state space. \emph{Visited states} are states the system entered at least once during the entire experiment. \emph{Reachable states} are states reachable by the policy during that episode. \emph{Cumulative signals} are the total feedback signals received by the end of each episode. Results shown are the mean and standard error over 10 trials.}
    \label{fig: competence_commparison.png}
\end{figure*}

\emph{Autonomous Vehicle Obstacle Passing}  An autonomous vehicle (AV) driving on a single-lane road has its lane blocked by an obstacle that can be stopped (e.g. a parked car) or moving (e.g. a large slow-moving tractor) and is blocking the AV from progressing. To get around the obstacle, the AV must cross the center lines and drive into the oncoming traffic's lane; the AV must reason about oncoming vehicles that it may not be able to see until it edges into the oncoming lane, as well as vehicles behind it that may move up to occupy its space when the AV enters the oncoming lane, at which point the AV cannot reverse. If the agent moves into the oncoming lane with an oncoming vehicle that is too close, the oncoming vehicle does not have time to stop and will crash into the AV. However, the oncoming vehicle may also choose to stop and wait for the AV to take its turn and pass; while the AV can determine if the oncoming vehicle is stopped, it must learn that it is allowed to pass when the oncoming vehicle is waiting for it. Similarly, when the vehicle behind it takes its space behind the obstacle and it cannot reverse, the AV must learn to discriminate between situations that require it to transfer control to a human and situations in which it can operate autonomously. $\kappa(s,a)$ is initialized to be $\{0,1\}$ in states where it has not entered into the oncoming lane as it must get explicit permission to attempt to do so, and $\{0,2\}$ in states where it is in the oncoming lane, as it is no longer safe to stop and wait for a response, and must instead rely on the human to override and take control when the situation becomes too dangerous.

\subsection{Campus Delivery Robot}
In the campus delivery robot domain, we conducted two experiments. In the first, the start and goal states remained static throughout all episodes resulting in the starvation of some (unvisited) state-action pairs. In the second, the start and goal states are randomly drawn from a collection of rooms throughout the campus map each episode so that no state-action pair is starved.

Figure \ref{fig:single_task_expected_cost} plots the mean and standard error of the expected cost of reaching the goal state in the single task experiment across episodes averaged over 10 trials. We see that, after an initial spike in the early episodes, the expected cost steadily decreases towards a steady state as the agent both learns to more accurately predict human feedback and updates its autonomy profile towards its competence.

The steady state in the expected cost is consistent with the results shown in Figure~\ref{fig:single_task_competence} where we see that the agent converges to almost 100\% level-optimality 
by episode 150 for all reachable states.  That is, the agent converges to a level-optimal solution that exploits its competence and no longer visits most states.
Hence the level-optimality stagnates at $\sim$90\% across all states and all visited states at this time as well.

Table \ref{tab:percent_actions_single_task} reports the percent of actions taken at each level of autonomy over the course of the first 150 episodes. We can see that by episode 150 the agent goes from requesting approval 45.8\% of the time in the first episode to operating in only no autonomy and unsupervised autonomy, indicating that the system has learned to properly exploit where it does or does not need to rely on the human authority.  Note that to accomplish this specific task, it is not cost-effective to request human assistance  in Levels 1 and 2, hence the agent learns to avoid them.  However, this is not a general pattern we would expect in every domain, but more of a reflection of the agent's competence and cost of human assistance in  Levels 1 and 2 in this domain.

\begin{figure}[t!]
    \centering
    \includegraphics[width=0.66\columnwidth]{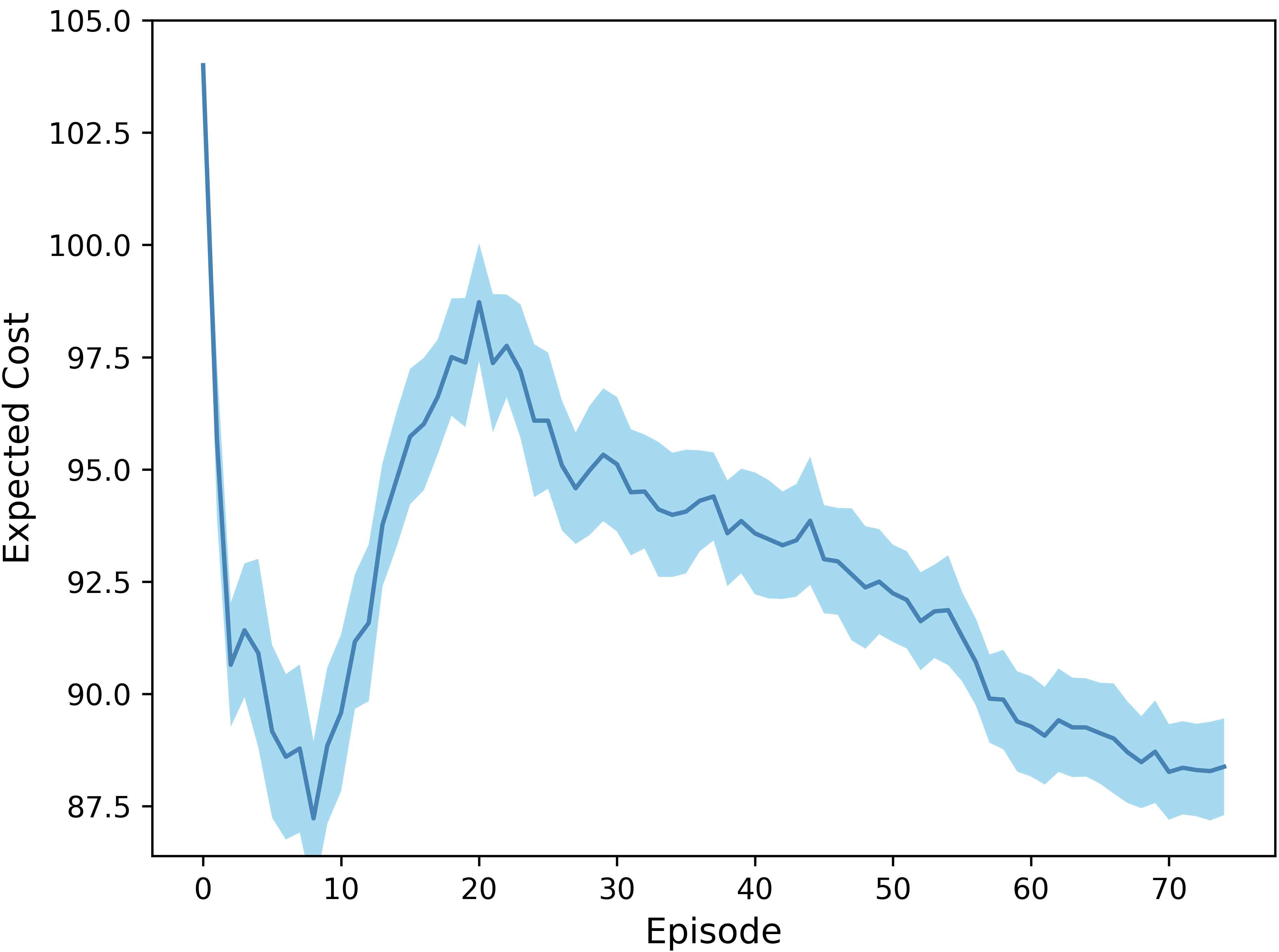}
    \vspace{-4pt}
    \caption{Campus delivery robot expected cost per episode.}
    \label{fig:single_task_expected_cost}
\end{figure}

\begin{table}[t!]
\center
\begin{tabular}{|rcccc|}
    \hline
    \textbf{Episode} & \textbf{Level 0} & \textbf{Level 1} & \textbf{Level 2} & \textbf{Level 3} \\ \hline \hline
    0       &   12.5  &   45.8  &   0.0   &   41.7  \\ \hline
    25      &   25.0  &   12.5  &   50.0  &   12.5  \\ \hline
    100     &   18.2  &   0.0   &   0.0   &   81.8  \\ \hline
    150     &   12.5  &   0.0   &   0.0   &   87.5  \\ \hline
\end{tabular}
\caption{Percent of actions taken at each level of autonomy in the single-task experiment.}
\vspace{-10pt}
\label{tab:percent_actions_single_task}
\end{table}

In the second experiment, the start and goal states are selected randomly in each episode, representing random tasks that the agent must complete throughout the entire campus. In this example, the agent covers a larger portion of the state space, but interacts less frequently with any given subset of states than it does in the first experiment. In such a setting, because no state-action pairs are starved in the limit, the agent will be guaranteed to converge to its competence in the limit across the entire environment; however because some states may be visited very infrequently, this convergence can take much longer than the 150 episodes it took to converge to 100\% in the first experiment across all reachable states. As a result, it is even more important that the agent generalizes what it has learned from prior experience and feedback.

Figure \ref{fig:random_task_competence} shows the results of this experiment. In particular we see that by episode 450, the system is operating at close to 100\% level-optimality across the entire state space. While this takes much longer than the 150 episodes it took to reach 100\% level-optimality across reachable states in the first experiment, we also see that by episode 150 in this experiment, the system is operating at $\sim$95\% level-optimality across the entire state space, almost five percent higher than in the first experiment, and using a similar number of feedback signals to reach this point. Because the system is able to optimize its reliance on the human authority across the entire environment, given a new task that it has not seen before, it would be able to perform close to optimally from the start.
 
Also shown in Figures \ref{fig:single_task_competence} and \ref{fig:random_task_competence} is the cumulative number of feedback signals the system has received by the end of each episode. A natural concern with a system that relies on feedback from a human is that the process of acquiring that feedback may be too onerous on the human to make it worthwhile. However, we can see in these figures that the system only requires $\sim$275 and $\sim$375 feedback signals respectively to reach near 100\% level-optimality, with most of the feedback coming from early episodes where the human may be expecting to be more actively involved. Furthermore, in our case all feedback comes as a result of assistive actions, and all assistive actions produce feedback; no feedback is given simply to provide information. This means that we are not introducing more work for the human supervisor to train the system, but simply utilizing existing information generated in the course of normal operation of the semi-autonomous system.

Finally, all results shown are the mean values over 10 trials along with standard error bars. We can see from the graphs that the standard error is consistently small across episodes indicating that there is small variance in the performance of the system, particularly with respect to the number of feedback signals the system needs to reach level-optimality in this domain.

\subsection{Autonomous Vehicle Obstacle Passing}
In the third experiment, an autonomous vehicle (AV) must navigate around an obstacle that is blocking its lane by moving into the oncoming lane, which may or may not have traffic. In doing so, the AV opens the space behind the obstacle to potentially be filled by a rear vehicle preventing it from moving back into the space if needed. It is worth noting that due to the highly safety-critical nature of this domain, we restrict the system's capacity to generalize both the feedback it receives and updates made to its autonomy profile. As an example, the AV being approved to edge into the oncoming lane when there is an oncoming vehicle far away has no impact on the likelihood of approval for edging when there is a vehicle some other distance away, or no observed vehicle at all. However, the system is allowed to generalize what it learns across certain features such as specific map location so it can apply what it learns to other single-lane locations.

\begin{figure}[t!]
    \centering
    \includegraphics[width=0.66\columnwidth]{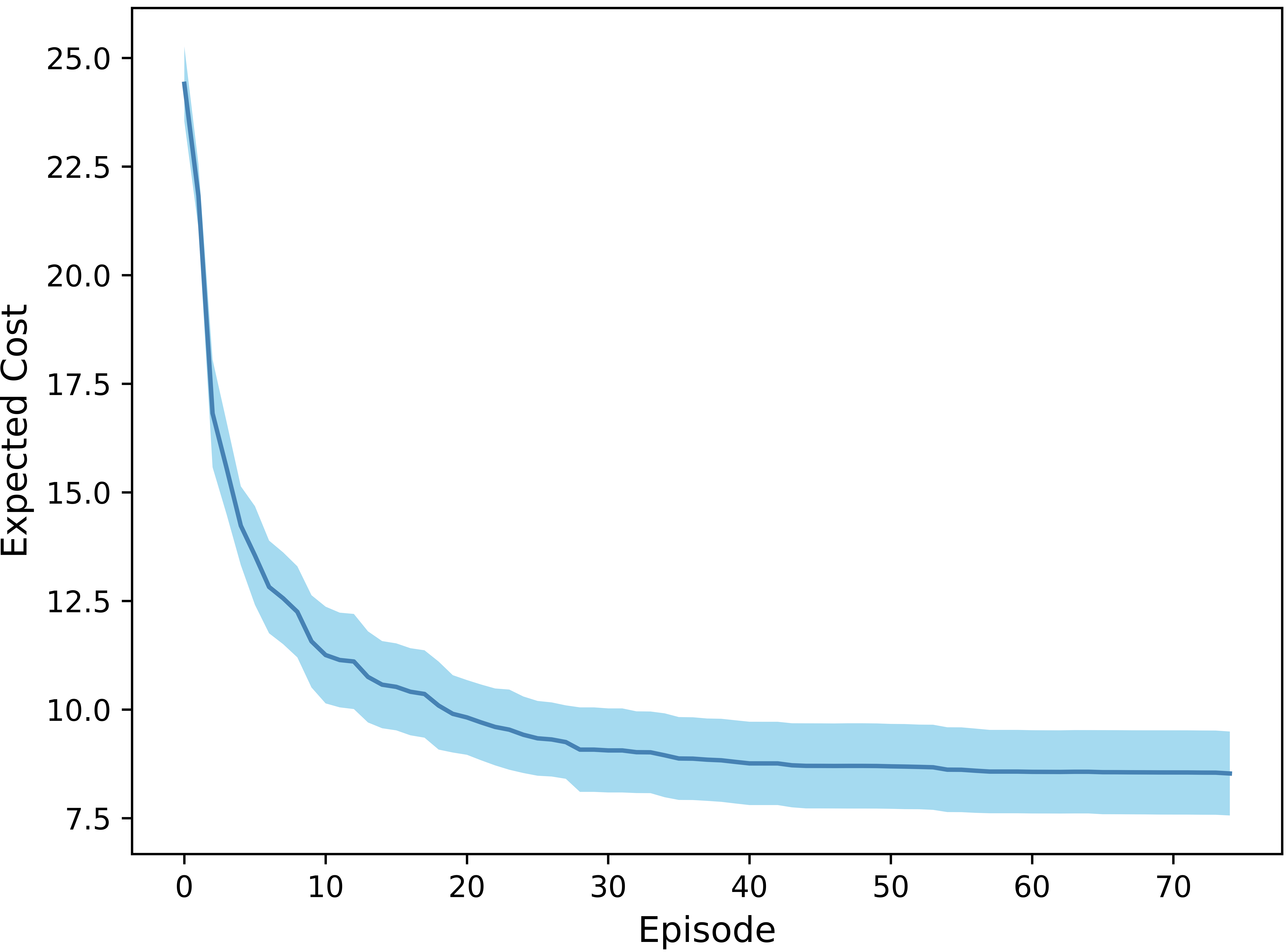}
    \caption{Obstacle passing domain expected cost per episode.}
    \label{fig:obstacle passing expected_cost}
\end{figure}

Despite this limitation, because the system interacts more consistently with a larger portion of the state space in each episode, the results we get are consistent with those in the campus delivery robot domain. As seen in Figure \ref{fig:obstacle passing expected_cost}, the expected cost decreases consistently with each episode as the AV learns to predict the human feedback better and its autonomy profile is updated towards its competence. In this problem specifically, it so happened that the system learned that the optimal policy for navigating the obstacle featured actions solely taken in unsupervised autonomy. Hence at the point where we see its level-optimality plateau in Figure \ref{fig:obstacle_passing_competence}, the system is no longer receiving information from the human as it is exploiting its autonomy model by only taking actions that it can do fully unsupervised to navigate the obstacle.

Similar to the campus delivery domain, the level-optimality plateaus at $\sim$90\% across both the entire state space, and states visited at least once throughout the experiment. This indicates that $\sim$10\% of states in the state space were either never visited, or visited too rarely for the agent to ever learn enough to converge to level-optimality in them. More importantly, the system is able to achieve almost 90\% level-optimality with fewer than 60 total feedback signals, most of which are obtained within the first 20 episodes. This is desirable when dealing with a safety critical problem domain that may be encountered infrequently in practice.

\section{Conclusion}
We present a new formal model -- competence-aware system -- for decision making in semi-autonomous systems where an autonomous agent can operate at different levels of autonomy, and must optimize its autonomous operation by learning from direct feedback signals provided by a human authority.
We demonstrated empirically that the approach enables the system to quickly learn to operate at its underlying competence in almost all situations, effectively optimizing its autonomous operation so as to minimize unnecessary reliance on human intervention. We validated the benefits of CAS across two simulated domains -- a campus delivery robot and an autonomous vehicle -- and obtained consistent results in all three of our experiments.

We provide theoretical results, including a new exploration rule, gated exploration, for enabling a CAS to safely explore new levels of autonomy in the presence of a human authority, and a proof that under normal convergence assumptions, the system is guaranteed to converge to its true competence and become level-optimal. Moreover, we showed that this result holds regardless of how the feedback profile is initialized.

Future work includes real-world experiments on an autonomous vehicle prototype with humans. We will also explore how we can utilize human feedback to identify potential features missing from the system's world model, as well as using our model to identify situations of low autonomy and over-reliance on a human authority.  Finally, we will examine ways to bootstrap competence learning to facilitate fast convergence when agent capabilities are modified. 

\begin{acks}
We thank the reviewers for their helpful comments. This work was supported in part by the National Science Foundation grant IIS-1724101 and in part by the Alliance Innovation Lab Silicon Valley.
\end{acks}

\makeatletter
\interlinepenalty=10000
\bibliographystyle{ACM-Reference-Format}
\bibliography{Baamas20}
\makeatother

\end{document}